\def\year{2018}\relax
\documentclass[letterpaper]{article} %DO NOT CHANGE THIS
\usepackage{aaai18}  %Required
\usepackage{times}  %Required
\usepackage{helvet}  %Required
\usepackage{courier}  %Required
\usepackage{url}  %Required
\usepackage{graphicx}  %Required
\usepackage{url}            % simple URL typesetting
\usepackage{booktabs}       % professional-quality tables
\usepackage{amsfonts}       % blackboard math symbols
\usepackage{nicefrac}       % compact symbols for 1/2, etc.
\usepackage{microtype}      % microtypography
\usepackage{color}

\usepackage{scalerel}
\usepackage{url}
\usepackage{bbold}
\usepackage{amsfonts}
\usepackage{amsmath,amssymb}
\usepackage{graphicx}
\usepackage{wrapfig}
\usepackage{mathrsfs} 
\usepackage{algorithm,algorithmic}
\usepackage{array}
\usepackage{times}
\usepackage{url}
\usepackage{upgreek}
\usepackage{amsthm}
\usepackage{float}
\usepackage{longtable}

\newcommand\numberthis{\addtocounter{equation}{1}\tag{\theequation}}

\newcommand{\R}{\mathbb{R}}

\newcommand{\E}{\mathbb{E}}

\newcommand{\Rb}{\mathbf{R}}

\newcommand{\ib}{\mathbf{i}}

\newcommand{\Xc}{\mathcal{X}}

\newcommand{\Yc}{\mathcal{Y}}

\newcommand{\Nc}{\mathcal{N}}
\newcommand{\Fc}{\mathcal{F}}
\newcommand{\Ec}{\mathcal{E}}

\newcommand{\argmin}{\text{argmin}}

\newcommand{\norm}[1]{\left\lVert#1\right\rVert}

\newcommand{\abs}[1]{\left|#1\right|}
\newcommand{\ind}[1]{\mathbf{1}\left(#1\right)}

\newtheorem{theorem}{Theorem}

\newtheorem{assumption}{Assumption}

\newtheorem{corollary}[theorem]{Corollary}
\newtheorem{definition}{Definition}

\newtheorem{proposition}[theorem]{Proposition}

\begin{document}
\title{On Data-Dependent Random Features for Improved Generalization\\ in Supervised Learning}

\author{Shahin Shahrampour, Ahmad Beirami, Vahid Tarokh \\ 
School of Engineering and Applied Sciences, Harvard University\\
Cambridge, MA,
02138 USA}

\maketitle
\begin{abstract}
The randomized-feature approach has been successfully employed in large-scale kernel approximation and supervised learning. The distribution from which the random features are drawn impacts the number of features required to efficiently perform a learning task. Recently, it has been shown that employing data-dependent randomization improves the performance in terms of the required number of random features. In this paper, we are concerned with the randomized-feature approach in supervised learning for good generalizability. We propose the Energy-based Exploration of Random Features (EERF) algorithm based on a data-dependent score function that explores the set of possible features and exploits the promising regions. We prove that the proposed score function with high probability recovers the spectrum of the best fit within the model class. Our empirical results on several benchmark datasets further verify that our method requires smaller number of random features to achieve a certain generalization error compared to the state-of-the-art while introducing negligible pre-processing overhead. EERF can be implemented in a few lines of code and requires no additional tuning parameters.
\end{abstract}

\section{Introduction}
At the heart of many machine learning problems, kernel methods (such as Support Vector Machine (SVM) \cite{cristianini2000introduction}) describe the nonlinear representation of data via mapping the features to a high-dimensional feature space. Even without access to the explicit form of the {\em feature maps}, one can still compute their inner products inexpensively using a kernel function, an idea known as the ``kernel trick''. However, unfortunately, methods using kernel matrices are not applicable to large-scale machine learning as they incur a prohibitive computational cost scaling at least quadratically with data. This observation motivated \cite{rahimi2007random} to consider kernel approximation using {\it random features}, and extend the idea to train shallow architectures \cite{rahimi2009weighted}. Replacing the optimization of nonlinearities by randomization, randomized shallow networks efficiently approximate the function describing the input-output relationship via random features. Nevertheless, a natural concern is the stochastic oracle from which the features are sampled. 
As noted in \cite{yang2012nystrom}, the basis functions used by random Fourier features \cite{rahimi2009weighted} are sampled from a distribution that is {\it independent} of the training set, and hence, a large number of random features may be needed to learn the data subspace. Therefore, one can ask whether {\em data-dependent} sampling can improve the prediction accuracy in supervised learning.

Recently, \cite{sinha2016learning}  proposed a data-dependent sampling scheme using an optimization perspective that can reduce the number of random features needed for effective learning. The generalization performance of their method, however, relies on the regularization parameter of the optimization problem, which requires an extra level of tuning (e.g., using cross-validation). 

In this paper, within the framework of supervised learning, we develop a data-dependent sampling method, called Energy-based Exploration of Random Features (EERF), with the goal of better generalizability.  Our algorithm operates based on a score function that is defined with respect to (a subsample of) the training samples. The algorithm explores the domain of random features, evaluates the score function in different regions, and outputs the promising random features for generalization. We prove that the score function mimics the spectrum of the best fit within the model class with high probability. We further apply our results to practical datasets, where we observe that our algorithm learns the subspace faster than the state-of-the-art as a function of the number of random features. Notably, our algorithm does not require additional parameter tuning.

{\bf \noindent Related literature on random features:} 
Some previous works on random features have focused on kernel approximation as well as prediction in a supervised manner. It has been shown that a wide variety of kernels  can be approximated efficiently using random features. Examples include shift-invariant kernels using Monte Carlo \cite{rahimi2007random} and Quasi Monte Carlo \cite{yang2014quasi} sampling, polynomial kernels \cite{kar2012random}, additive kernels \cite{vedaldi2012efficient}, and many more. In particular, Gaussian kernel has received considerable attention (see e.g. \cite{felix2016orthogonal} for a recent study on efficient Gaussian kernel approximation), and the error of random Fourier features has been analyzed in the context of kernel approximation \cite{sutherland2015error}. Additionally, the generalization property of the randomized-feature approach has been theoretically studied from the statistical learning theory viewpoint \cite{rudi2016generalization}. 

Another line of research has focused on decreasing the time and space complexity of kernel approximation. In \cite{le2013fastfood}, the Fast-food method has been developed to approximate kernel expansions in log-linear time. The underlying idea is that Hadamard matrices combined with diagonal Gaussian matrices exhibit properties similar to dense Gaussian random matrices. Using this approach a class of flexible kernels has been proposed in \cite{yang2015carte}.

Data-dependent random features have been recently studied in \cite{yu2015compact,oliva2016bayesian,chang2017data} for approximation of shift-invariant kernels. We consider a broader class of kernels (see Eq. \eqref{kernel}) and propose a sampling scheme that improves the generalization error, particularly when the number of random features is small. Our work is particularly relevant to that of \cite{sinha2016learning}, where a data-dependent optimization approach is developed to sample features with promising generalization error for small to moderate number of bases. Their generalization performance relies on the regularization parameter in their optimization problem, which requires an extra level of tuning (e.g., using cross-validation). Our method, in contrast, does not need additional parameter tuning. Finally, the details of the benchmark algorithms used for comparison in this paper can be found  in Table \ref{table10} in the empirical evaluations section.

{\bf \noindent Nystr{\"o}m method:} 
We remark that Nystr{\"o}m method \cite{williams2001using,drineas2005nystrom} adopts an alternative viewpoint for approximation of kernel by a low rank matrix. The method samples a subset of training data, approximates a kernel matrix, and then transforms the data using the approximated kernel. We refer the reader to \cite{yang2012nystrom} for a discussion on the fundamental differences between the Nystr{\"o}m method and random features.
 
{\bf \noindent A note on multiple kernel learning (MKL):} 
The goal of MKL is to learn a good kernel based on training data (see e.g. \cite{gonen2011multiple} for a survey). For the supervised learning setup, various methods consider optimizing a convex, linear, or nonlinear combination of base kernels with respect to a measure (e.g. kernel alignment) to identify the ideal kernel \cite{kandola2002optimizing,cortes2009learning,cortes2012algorithms}. Another approach is to optimize the kernel and the empirical risk simultaneously \cite{kloft2011lp,lanckriet2004learning}. These methods enjoy various theoretical guarantees \cite{bartlett2002rademacher,cortes2010generalization}, but they involve costly computational steps, such as eigen-decomposition of the Gram matrix (see \cite{gonen2011multiple} for details). The distinction of our work with this literature is that we do not consider a combination of base kernels. Instead, we propose to use the randomized-feature approach of \cite{rahimi2009weighted,rahimi2007random} with a data-dependent sampling scheme to avoid computational cost.

{\noindent \bf Notation:}
We denote by $\Nc(\mu,\sigma^2)$ the Gaussian distribution with mean $\mu$ and variance $\sigma^2$, by $I_d$ the identity matrix of size $d$, and by $[N]$ the set of positive integers $\{1,\ldots,N\}$, respectively. $w_i$ is the $i$-th component of the vector $w$, whereas $z^m$ is $m$-th sample among the batch $\{z^j\}_{j=1}^M$.
\section{Problem Setup}
Consider the supervised learning setup: we are given a training {\it input-output} set $\{(x^n,y^n)\}_{n=1}^N$, where the pairs are generated independently from an {\it unknown, fixed} distribution $P_{\Xc\Yc}$, and for every $n\in [N]$, $x^n=[x^n_1\cdots x^n_d]^\top \in \Xc \subset \R^d$. For regression, the response variable $y^n \in \Yc \subseteq [-1,1]$, while for classification  $y^n \in \{-1,1\}$. The goal is to fit a function $f: \Xc  \to \R$ to training data via risk minimization. As $P_{\Xc\Yc}$ is not available, we consider minimizing the empirical risk $\widehat{\Rb}(f)$ in lieu of $\Rb(f)$,
\begin{align}
\resizebox{1\hsize}{!}{$% 
\Rb(f)\triangleq\E_{P_{\Xc\Yc}}[c(f(x),y)], \quad \widehat{\Rb}(f)\triangleq\frac{1}{N}\sum_{n=1}^N c(f(x^n),y^n),\label{emprisk}
$%
}%
\end{align}
where $c(\cdot,\cdot)$ is a task-dependent loss function (e.g., quadratic for regression, hinge loss for SVM), measuring the dissimilarity of the mapping $f(x)$ and the output $y$ over training samples. In general, one often parametrizes the function $f(\cdot)$ to minimize $\widehat{\Rb}(f)$ over a parameter space. Kernel methods offer such solutions where $f(x)\approx\sum_{n=1}^N \alpha_nk(x^n,x)$, and the empirical risk $\widehat{\Rb}(f)$ is minimized over $\{\alpha_n\}_{n=1}^N$. However, an immediate drawback of this approach is its inapplicability to large-scale data, since the computational complexity scales at least with $O(N^2)$ for training the kernel matrix. To overcome this shortcoming, an elegant approach was proposed by \cite{rahimi2009weighted}, where shallow networks are parametrized using {\it random features}. Let us represent a feature map by $\phi: \Xc \times \Omega \to \R$, where $\Omega$ is the support set for random features. Then, considering functions of the form
\begin{align}\label{mapping}
f(x)=\int_\Omega F(\omega)\phi(x,\omega)d_\omega,
\end{align}
one can approximate $f(x)\approx \sum_{m=1}^M \theta_m \phi(x,\omega^m)$ and minimize $\widehat{\Rb}(f)$ over $\{\theta_m,\omega^m\}_{m=1}^M$, where $M$ is (hopefully) much smaller than $N$. The main issue would be the joint optimization of $\{\theta_m,\omega^m\}_{m=1}^M$, which results in a non-convex problem. \cite{rahimi2009weighted} showed that one can randomize over $\{\omega^m\}_{m=1}^M$ (the so-called randomized-feature approach) and minimize $\widehat{\Rb}(f)$ only on $\{\theta_m\}_{m=1}^M$, which is an efficiently solvable convex problem. The resulting solution was shown to be not much worse than the solution obtained by optimally tuning $\{\omega^m\}_{m=1}^M$. In this context, the feature map can be an eigenfunction of a positive-definite kernel, and for any $x^n,x^{n'}\in \Xc$, the kernel can be represented as
\begin{align}\label{kernel}
K_{P_\Omega}(x^n,x^{n'})=\int_\Omega \phi(x^n,\omega)\phi(x^{n'},\omega)P_\Omega(\omega)d_\omega,
\end{align}
where $P_\Omega$ is a distribution on the random features\footnote{More accurately, $P_\Omega$ is a probability density function when $\Omega$ is continuous, whereas it is a probability mass function when $\Omega$ is discrete, but instead, we use the word distribution to refer to both.}. Using various feature maps and distributions one can recover commonly used kernels (e.g. Gaussian, Cauchy, Laplacian, arc-cosine, and linear) from \eqref{kernel}. A list of possible choices can be found in Table 1 of \cite{yang2014random}.

The approach of \cite{rahimi2009weighted} is particularly appealing due to its computational tractability since preparing the feature matrix during training requires $O(MN)$ computations, while evaluating a test sample needs $O(M)$ computations, which significantly outperforms the complexity of traditional kernel methods. However, the potential drawback is that random features are drawn from a distribution $P_\Omega$, {\it independent} of the training set, and therefore, we may require a large number of random features before learning the data subspace~\cite{yang2012nystrom}. 

More specifically, under mild assumptions, the algorithm proposed in \cite{rahimi2009weighted} outputs an approximation $\widehat{f}(\cdot)$, which given a sampling distribution $P_\Omega$, with probability at least $1-\varepsilon$ satisfies,
\begin{align} \label{recht}
\resizebox{1\hsize}{!}{$% 
\Rb(\widehat{f}~)-\min_{f\in \Fc_{P_\Omega}}\Rb(f) \leq O\left(\left(\frac{1}{\sqrt{M}}+\frac{1}{\sqrt{N}}\right)C\sqrt{\log \varepsilon^{-1}}\right),
$%
}%
\end{align}
where
\begin{align}\label{class}
\resizebox{1\hsize}{!}{$%
\Fc_{P_\Omega}\triangleq\{f(x)=\int_\Omega F(\omega)\phi(x,\omega)d_\omega : \abs{F(w)} \leq CP_{\Omega} (\omega) \}.
$%
}%
\end{align}
 As noted by \cite{rahimi2009weighted}, $\Fc_{P_\Omega}$ is a rich class consisting of functions whose weights decay faster than the given sampling distribution $P_{\Omega}$. As an example, in the case of sinusoidal feature maps, the set comprises of functions whose Fourier transforms decay faster than $CP_{\Omega}$. This intuitively implies that given a sampling distribution $P_0$ (e.g., Gaussian) and a large constant $C_0$, we can hope to (under mild technical assumptions on the target function) push the best candidate within the class $\Fc_{P_0}$ to the target function. Given such $P_0$ and $C_0$, let us assume that the best function fit within the model class $\Fc_{P_0}$ is
\begin{align}\label{bestmodel}
f_0(x)\triangleq \argmin_{f\in \Fc_{P_\Omega}}\Rb(f)=\int_\Omega F_0(\omega)\phi(x,\omega)d_\omega.
\end{align}
While using the pair $(C_0,P_0)$, $f_0(\cdot)$ can eventually be recovered (due to \eqref{recht}), we may as well try to modify the initial sampling distribution $P_0$ according to the shape of $F_0(\cdot)$. In other words, if we knew $F_0(\cdot)$ precisely, we could set $P_{\Omega}(\omega)=\frac{\abs{F_0(\omega)}}{\int_\Omega \abs{F_0(\omega')}d_{\omega'}}$ and $C=\int_\Omega \abs{F_0(\omega')}d_{\omega'}$, respectively, to sample the randomized features that weight more in the spectrum $F_0(\cdot)$.
 In the next section, we propose an algorithm that exploits the training data to find a ``good" set of random features; thereby, improving the approximation of $f_0(\cdot)$ using finitely many random features for better generalization with small number of random features. The key is to use a {\it data-dependent} score function that approximately mimics the shape of $F_0(\cdot)$ with some error.

%%%%%%%%%% EERF Section

\section{Proposed Method: Energy-based Exploration of Random Features (EERF)}

In this section, we propose an algorithm to choose random features that maintain a low generalization error in supervised learning. The algorithm employs a {\it score} function to explore the domain of random features and retains the samples with the highest score. The key is to use a proper score function $S: \Omega \to \R$, which we define to be 
\begin{align}\label{score}
S(\omega)\triangleq\E_{P_{\Xc\Yc}}[y\phi(x,\omega)]   \ \  \ \ \  \ \ \widehat{S}(\omega)\triangleq\frac{1}{N}\sum_{n=1}^N y^n\phi(x^n,\omega),
\end{align}
where $\widehat{S}(\cdot)$ denotes its empirical estimate. The score function can mimic kernel polarization \cite{baram2005learning} asymptotically in the limit of large $M$. In particular, $\frac{1}{M}\sum_{m=1}^{M}\widehat{S}^2(\omega^m)$ for an i.i.d. sequence $\{\omega^m\}_{m=1}^M$ amounts to kernel polarization, which aims to polarize the data in the associated feature space to draw correspondence between
the proximity of the points in the high-dimensional feature space and their responses. 
Roughly speaking, $S^2(\omega)$ can be considered to be an energy spectral density, after which the proposed algorithm is named.
As is formally stated in Theorem~\ref{thm}, the proposed score function $S(\omega)$ is aligned to the spectrum $F_0(\omega)$ (up to an inevitable projection error). Given this result, we can use $\widehat{S}(\cdot)$, the empirical version of the true score, to re-weight features and modify the initial data-independent sampling distribution. Before further investigation of the behavior of the score function, we describe our algorithm.

{\noindent \bf Algorithm:}
Our algorithm works as follows: it draws $M_0$ samples from an initial distribution $P_0$, evaluates them in the empirical score given in \eqref{score}, and selects the top $M$ samples  in the sense of maximizing $|\widehat{S}(\cdot)|$. The pseudo-code is given in Algorithm \ref{ALGO1}.
\begin{algorithm}[h!]
	\caption{Energy-based Exploration of Random Features (EERF)}
	{\bf Input:} $\{(x^n,y^n)\}_{n=1}^{N}$, the feature map $\phi(\cdot,\cdot)$, integers $M_0$ and $M$ where $M\leq M_0$, initial sampling distribution $P_0$.\\ 
	\begin{algorithmic}[1]\label{ALGO}
			\STATE Draw samples $\{\tilde{\omega}^m\}_{m=1}^{M_0}$ independently from $P_0$. 
			\STATE Evaluate the samples in $\widehat{S}(\cdot)$, the empirical score in \eqref{score}.
			\STATE Sort $|\widehat{S}(\cdot)|$ for all $M_0$ samples in descending order, and let $\{\omega^m\}_{m=1}^{M}$ be the top $M$ arguments, i.e., the ones that give the top $M$ values in the sorted array.
	\end{algorithmic}
	{\bf Output:} $\{\omega^m\}_{m=1}^{M}$.
\label{ALGO1}
\end{algorithm}

Once we have the ``good'' $M$ features $\{\omega^m\}_{m=1}^{M}$, we can solve the following empirical risk minimization \cite{rahimi2009weighted}
\begin{align}\label{risk-recht}
\resizebox{0.9095\hsize}{!}{$%
\widehat{\theta}=\underset{\theta: \norm{\theta}_\infty \leq \frac{C}{M}}{\argmin} ~~~ \left\{\frac{1}{N}\sum_{n=1}^N c\left(\frac{1}{\sqrt{M}}\sum_{m=1}^M \theta_m \phi(x^n,\omega^m),y^n\right)\right\},
$%
}%
\end{align}
to approximate the underlying model. We remark that the EERF algorithm requires $O(dNM_0)$ computations to calculate the empirical score and $O(M_0\log M_0)$ time on average to sort the $M_0$ obtained scores. One can often use a subsample of the training set instead of the entire $N$ samples, and the value of $M_0$ should be set to an integer multiple of $M$. The initial distribution $P_0$ is either trivial to choose (e.g. uniform for the linear kernel, or standard Gaussian for the arccosine kernel), or can be selected using some rules-of-thumb. We elaborate on these issues in the experiments. 

{\noindent \bf Theoretical results:}
The key to understanding Algorithm \ref{ALGO1} is to analyze the empirical score $\widehat{S}(\cdot)$. It is immediate from McDiarmid's inequality that with probability at least $1-\delta$, the empirical score is concentrated around the true score as 
$$
\abs{\widehat{S}(\omega)-S(\omega)} \leq O\left(\sqrt{\frac{ 2\log\frac{M_0}{\delta}}{N}}\right),
$$ 
for all samples $\{\tilde{\omega}^m\}_{m=1}^{M_0}$. Restricting our attention to the regression model in Theorem \ref{thm}, we show that $F_0(\cdot)$ and $S(\cdot)$ (when normlizad) exhibit similar behavior. That is, we can use the empirical version of $S(\cdot)$ in lieu of the {\it unknown} spectrum $F_0(\cdot)$ to better approximate $f_0(\cdot)$. An informal version of our result can be stated as follows

\begin{theorem}\label{thm}
For the regression model where $\E_{P_{\Yc \vert \Xc}}[y]=f^\star(x)$, under some technical assumptions, we have
\begin{equation}\label{theo}
\frac{\abs{S(\omega)-\text{err}_{\text p}(\omega)}}{\int_\Omega \abs{S(u)-\text{err}_{\text p}(u)}d_u}\approx \frac{\abs{F_0(\omega)}}{\int_\Omega \abs{F_0(u)}d_u},
\end{equation}
where $\text{err}_{\text p}(\cdot)$ is bounded by the sup-norm of $f^\star(\cdot)-f_0(\cdot)$.
\end{theorem}
The exact statement of the theorem (including assumptions) and its proof and consequences are given in the extended version of the paper~\cite{AAAI-arZiv}. Note that the projection error is inevitable and is a defect of the model class, and not the algorithm. The content of the theorem is that the score function aligns with the ``spectrum'' of the best model in the model class \eqref{class} (up to some projection error). Following the discussion after \eqref{bestmodel}, recall that the right-hand side of \eqref{theo} is precisely what we are seeking for the reconstruction of $f_0(\cdot)$, and our algorithm calculates an empirical version of $S(\cdot)$ to approximate the right-hand side of \eqref{theo}. Moreover, as we shall find in the supplementary material, $\textit{err}_{\textit p}(\cdot)$ is a decreasing function of $C$  in \eqref{class}, i.e., by increasing $C$, we make the class $\Fc_{P_\Omega}$ richer and decrease the projection error. 

We remark that although the focus of Theorem \ref{thm} is on the regression model, the same approach intuitively applies to the logistic regression model for classification. In logistic regression, it can be shown that $\E_{P_{\Yc \vert \Xc}}[y]=\tanh(f^\star(x)/2)$. Observe that $|\tanh(z/2)|$ is a monotonic function of $|z|$, and hence, the selection of random features based on the score function still aligns with the spectrum of the best model within the model class  \eqref{class}. Thus, the underlying intuition used to describe polarization (alignment) after \eqref{score} still holds.

%%%%%%%%%% End of EERF Section

%\renewcommand{\arraystretch}{2} 
\begin{table*}[t!]
\caption{We compare our work to the baselines in the left-most column. The right-most column lists the prior art (on random features) that was compared against in the baseline paper.}
\begin{center}
\resizebox{2.1\columnwidth}{!}{
\begin{tabular}{| c ||  c  | c | c | @{}m{0pt}@{}} 
 \hline 
Baseline   & Data-dependent   & Class of kernels & Prior work used for comparison by the baseline &\\ [3.25 ex]
\hline
RKS \cite{rahimi2009weighted} & No & Eq. \eqref{kernel} & Adaboost &\\ [3.25 ex]
\hline
ORF \cite{felix2016orthogonal} & No & Gaussian &RFF \cite{rahimi2007random}, Fastfood \cite{le2013fastfood}, QMC \cite{yang2014quasi}, Circulant \cite{yu2015compact} &\\ [3.25 ex]
\hline
LKRF \cite{sinha2016learning} & Yes & Eq. \eqref{kernel} & RKS \cite{rahimi2009weighted}&\\ [3.25 ex]
\hline
SES \cite{chang2017data}   & Yes   & Gaussian &  RFF \cite{rahimi2007random}, BQ \cite{huszar2012optimally}, QMC \cite{yang2014quasi}  &\\ [3.25 ex]
\hline
\end{tabular}\label{table10}}
\end{center}
\end{table*}

\section{Empirical Results}
\subsection{Gaussian Kernel}\label{sec:simulations}
We apply our proposed method to several datasets from the UCI Machine Learning Repository. Since all of our baseline algorithms are applicable to Gaussian kernels, we first compare our method to the state-of-the-art within that framework, and next we show the applicability of our method to linear and arccosine kernels.

{\noindent \bf Benchmark algorithms:}
We use the algorithms in Table \ref{table10} as baselines for comparison. Notice that in Table \ref{table10}, we have also reported the prior work used by each baseline for comparison. The following comments are in order: 
\begin{itemize}
\item The ORF algorithm involves a QR decomposition step ($O\left(d^3\right)$ time) which can be side-stepped using the companion algorithm SORF in \cite{felix2016orthogonal}. The main advantage of SORF, which combines Walsh-Hadamard matrices and diagonal ``sign-flipping'' matrices, is computational, thus when the prediction accuracy is concerned, SORF and ORF are shown to have similar performance, while SORF performs worse than ORF for $d<32$ \cite{felix2016orthogonal}.
\item LKRF introduces a pre-processing optimization to re-weight random features. The algorithm initially samples $M_0$ random features, forms the optimization with $O\left(dM_0N\right)$ computations, and requires $O\left(M_0\log\epsilon^{-1}\right)$ time to find an $\epsilon$-optimal solution. Also, the optimization involves a hyper-parameter balancing the trade-off between an alignment measure versus the $f$-divergence of solution with the uniform distribution. We run the algorithm multiple times with the hyperparameter in the set $\{10^{-5},10^{-4},\ldots,10^5\}$ and report the best result. 
\item The SES algorithm also re-weights random features by solving an optimization problem using sketching techniques. Letting $T_S$ be the time cost of sketching, the optimization problem costs $O\left(rd^2 +T_S\right)$, where $r$ is the number of samples included in the sketching matrix. The main purpose of SES is kernel approximation, but when applied to supervised learning on a number of datasets, SES has proven to be competitive to Monte Carlo and Quasi Monte Carlo methods (see Table \ref{table10}). 
\end{itemize}
Following \cite{rahimi2009weighted}, we replace the infinity-norm constraint of \eqref{risk-recht} by a quadratic regularizer in practice. We then tune the regularization parameter by trying different values from $\{10^{-5},10^{-4},\ldots,10^5\}$. For all methods (including ours) in the Gaussian case, we sample random features from $\frac{1}{\sigma}\Nc(0,I_d)$. The value of $\sigma$ for each dataset is chosen to be the mean distance of the $50$th $\ell_2$ nearest neighbor, which is shown to result in good classification\footnote{This choice is a rule-of-thumb and further tuning may result in improved generalization. However, we use the same choice for all of the methods for a fair comparison.} \cite{felix2016orthogonal}. It is also important to note that:
\begin{itemize}
\item RKS, ORF, and SES draw $M$ samples from the Gaussian distribution. RKS directly uses the samples, ORF ``orthogonalizes'' them to another set of $M$ vectors, and SES re-weights them using an optimization.
\item In contrast, LKRF and EERF (our work) draw $M_0>M$ samples from the Gaussian distribution, process them, and use the most promising $M$ samples. The value of $M_0$ for each dataset along with the pre-processing overhead for both algorithms are reported in Table \ref{table6}.
\end{itemize}

All codes are written in MATLAB and run on a machine with CPU 2.9 GHz and 16 GB memory.

\begin{figure*}[t]
\begin{center}
\includegraphics[width = \linewidth]{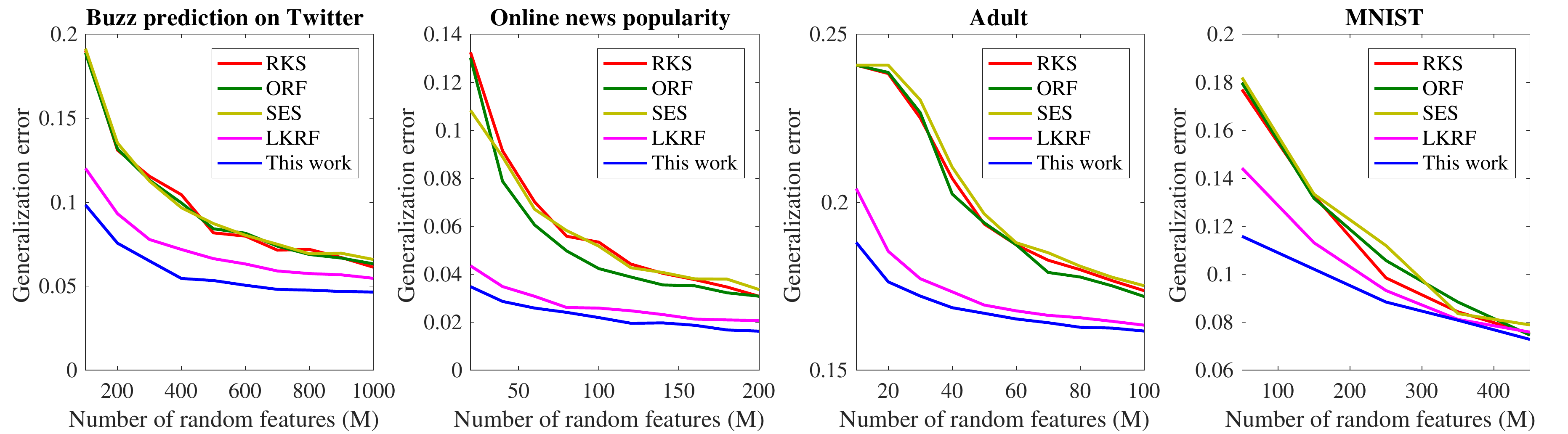}
\end{center}
	\caption{Performance on practical datasets: we compare the generalization error of our method (EERF) with the baselines RKS, ORF, SES, and LKRF. In all cases, for a fixed $M$, our algorithm achieves a smaller generalization error. }
	\label{plot1}
\end{figure*}

{\noindent \bf Datasets:}
Table \ref{table1} represents the number of training samples ($N_\text{train}$) and test samples ($N_\text{test}$) used for each dataset. If training and test sets are provided explicitly, we use them accordingly; otherwise, we split the dataset randomly. The features in all datasets are scaled to be empirically zero mean and unit variance and the responses in regression tasks are normalized to $[-1,1]$.
\begin{table}[h!]
\caption{The description of the datasets used for Gaussian kernel: the number of features, training samples, and test samples are denoted by $d$, $N_\text{train}$, and $N_\text{test}$, respectively.}
\begin{center}
\resizebox{1\columnwidth}{!}{
\begin{tabular}{| c ||  c | c | c | c | @{}m{0pt}@{}} 
 \hline 
  Dataset &  Task   & $d$ & $N_\text{train}$ & $N_\text{test}$ &\\ [1.5 ex]
 \hline  
 Buzz prediction on Twitter &  Regression  &  77  & 93800 & 46200 &\\ [1.5 ex]
 \hline
Online news popularity &  Regression &  58   & 26561 & 13083 &\\ [1.5 ex]
 \hline
Adult & Classification & 122     & 32561 & 16281  &\\ [1.5 ex]
 \hline
MNIST  & Classification  & 784    & 60000 & 10000  &\\ [1.5 ex]
\hline
\end{tabular}\label{table1}}
\end{center}
\end{table}

\begin{table*}[h!]
\caption{Comparison of the time cost and performance of our algorithm versus RKS. $t_{pp}$, $t_{\text{train}}$, and $t_{\text{train}}$ represent pre-processing, training, and testing time (seconds). $N_0$ is the number of samples we use for pre-processing, and $M_0$ is the number of random features we initially generate. $M$ is the number of random features used by both algorithms for eventual prediction. The standard errors are reported in parentheses.}
\begin{center}
\resizebox{2\columnwidth}{!}{
\begin{tabular}{| c ||  c | c | c | c | c | c| c |c |c | c| @{}m{0pt}@{}} 
 \hline 
  Dataset &  $M$   &  $M_0$ &  $N_0/N $ & Our $t_{pp}$ & Our $t_{\text{train}}$ & Our $t_\text{test}$ & RKS $t_\text{train}$ & RKS $t_\text{test}$ & Our error (\%) &  RKS error (\%) &\\ [1.5 ex]
 \hline  
 Buzz prediction on Twitter &  1000 &  10000 &  10\%  & 0.84  & 1.96 & 1.76  & 2.11 &1.78 & {\bf 4.65} (4e-2) & 6.09 (7e-2) &\\ [1.5 ex]
 \hline
Online news popularity &  200 & 20000 & 5\% & 0.53  & 0.15 & 0.07 & 0.13 & 0.04 &{\bf  1.63} (4e-2) & 3.08 (5e-2) &  \\ [1.5 ex]
 \hline
Adult & 100 &  2000     & 5\% & 0.19 & 1.78 & 0.05   &  1.61 & 0.06 & {\bf 16.16} (2e-2) & 17.37 (6e-2) &\\ [1.5 ex]
 \hline
MNIST  & 450  &  10000 & 20\%   & 5.20  & 16.17 &  8.58  & 19.45  & 10.65  & {\bf 7.28} (3e-2) & 7.53 (1.8e-1) &\\ [1.5 ex]
\hline
\end{tabular}\label{table2}}
\end{center}
\end{table*}

\begin{figure*}
\begin{center}
\includegraphics[width = \linewidth]{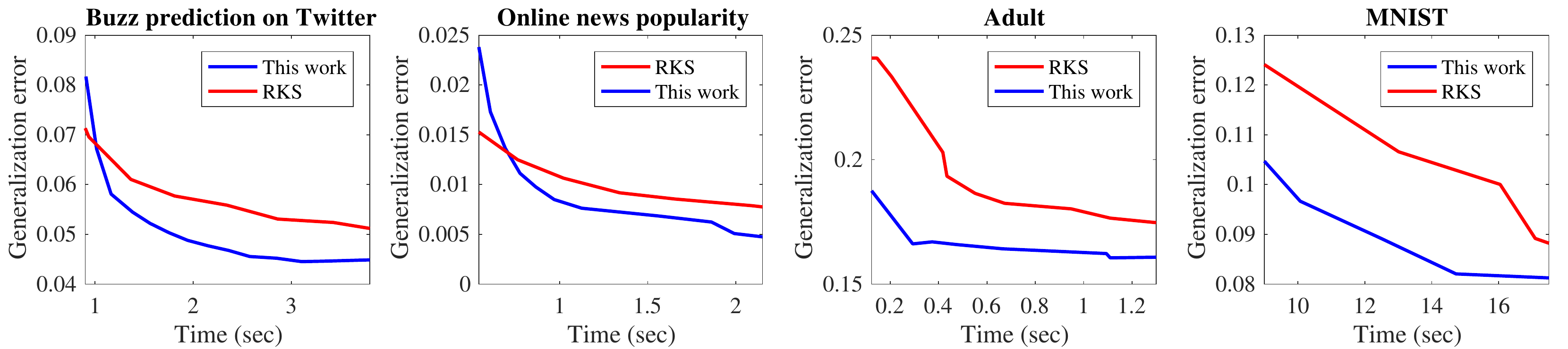}
\end{center}
	\caption{Generalization error versus time for our method against RKS. The time for our method is the summation of training and pre-processing time, whereas for RKS it is the training time.}
	\label{plot3}
\end{figure*}

\begin{figure*}[t]
\begin{center}
\includegraphics[width = 0.9\linewidth]{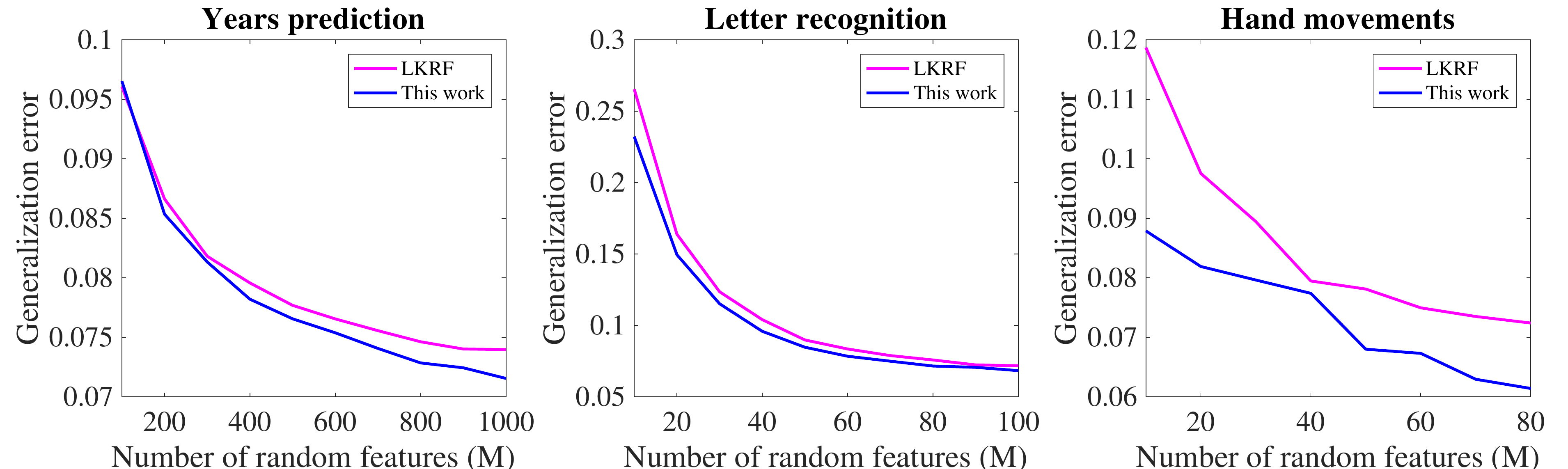}
\end{center}
	\caption{Performance on three datasets: we compare the generalization error of our method (EERF) against LKRF.}
	\label{plot2}
\end{figure*}

\begin{table}[h!]
\caption{The description of the datasets used for linear and arccosine kernels: the number of features, training samples, and test samples are denoted by $d$, $N_\text{train}$, and $N_\text{test}$, respectively. $H(\cdot)$ denotes the Heaviside step function ($H(x)=0.5+0.5\text{sgn}(x)$, where $\text{sgn}(\cdot)$ is the sign function).}
\begin{center}
\resizebox{1\columnwidth}{!}{
\begin{tabular}{| c ||  c | c | c | c | c | @{}m{0pt}@{}} 
 \hline 
  Dataset &  Task   & $\phi(x,\omega)$ &  $d$ & $N_\text{train}$ & $N_\text{test}$ &\\ [1.5 ex]
 \hline
Years prediction &  Regression  &  $(\omega^\top x)H(\omega^\top x)$ & 90   & 463715 & 51630 &\\ [1.5 ex]
 \hline  
Letter recognition &  Classification  & $(\omega^\top x)^2H(\omega^\top x)$  & 16  & 15000 & 5000 &\\ [1.5 ex]
 \hline
Hand movements &  Classification  & $x_\omega$ & 561 & 7352 & 2947 &\\ [1.5 ex]
 \hline
\end{tabular}\label{table3}}
\end{center}
\end{table}

\begin{table*}[h!]
\caption{Comparison of the time cost and performance of our algorithm versus LKRF. $t_{pp}$, $t_{\text{train}}$, and $t_{\text{train}}$ represent pre-processing, training, and testing time (sec). $N_0$ is the number of samples both algorithms use for pre-processing, and $M_0$ is the number of random features they initially generate. $M$ is the number of random features used by both algorithms for eventual prediction. The standard errors are reported in parentheses.}
\begin{center}
\resizebox{2.05\columnwidth}{!}{
\begin{tabular}{| c ||  c | c | c | c | c | c| c| c  | c| @{}m{0pt}@{}} 
 \hline 
  Dataset &  $M$   &  $M_0$ &  $N_0/N$ &  Our $t_{pp}$ & LKRF $t_{pp}$ &  Our $t_{\text{train}}$  & LKRF $t_{\text{train}}$& Our error (\%) & LKRF error (\%)& \\ [1.5 ex]
 \hline
 Buzz prediction on Twitter &  1000 &  10000 &  10\%  & 0.84  & 0.97 & 1.96 & 1.88 & {\bf 4.65} (4e-2) & 5.23 (6e-2)& \\ [1.5 ex]
\hline
Online news popularity &  200 & 20000 & 5\% & 0.53  &  0.50 &    0.15 & 0.14  &{\bf  1.63}  (4e-2) & 2.07 (5e-2) &  \\ [1.5 ex]
\hline
Adult & 100 &  2000     & 5\% & 0.19 & 0.08 & 1.78   &  1.35  & {\bf 16.16} (2e-2) & 16.34 (2e-2) &\\ [1.5 ex]
\hline
MNIST  & 450  &  10000 & 20\%   & 5.20  & 6.13  &  15.97 & 16.17 & {\bf 7.28} (3e-2) & 7.59 (1.7e-1) &\\ [1.5 ex]
\hline
Year prediction & 1000 & 4000 & 10\% &  6.23  & 7.24 & 32.31 & 53.96 & {\bf 7.15} (1.9e-2) & 7.40 (1.4e-2) &\\ [1.5 ex]
\hline
Letter recognition & 100 & 500 & 100\% & 4.55   & 5.44 & 11.33 & 12.95 & {\bf 6.83} (7e-2) & 7.17 (8e-2) &\\ [1.5 ex]
\hline
Hand movement & 80 & 561 & 100\% &  0.18  & 0.04 & 0.83 & 0.98 & {\bf 6.14} (2.7e-3) & 7.24 (1.2e-2) &\\ [1.5 ex]
\hline
\end{tabular}\label{table6}}
\end{center}
\end{table*}

{ \noindent \bf Performance:}
The results on datasets in Table \ref{table1} are reported in Fig. \ref{plot1}: for each dataset, by pre-processing random features in the score function \eqref{score}, our method learns the subspace faster compared to state-of-the-art, i.e., we require smaller number of random features $M$ to achieve a given generalization error threshold. As the number of samples increases, all methods tend to generalize better, which is not surprising, since they eventually sample the ``good'' random features for learning the data model. In the regime of moderate $M$, LKRF closely competes with our algorithm due to its data-dependent pre-processing phase. We will elaborate on the performance of our method versus LKRF in the next section, after experiments on linear and arccosine kernels are also presented. 

Table \ref{table2} tabulates the time cost and the generalization error for our method and RKS used with a fixed Gaussian kernel. For each dataset, the statistics are reported for the largest value of $M$ used in the experiment. We randomly subsample $N_0$ data points of the dataset to calculate the empirical score (e.g., for ``Buzz prediction on Twitter'' we use $10\%$ of the training samples) and generate an initial $M_0$ random features to evaluate the score function. Then, the most promising $M$ samples with the highest scores are selected following Algorithm \ref{ALGO1}, and the performance is compared to the case where $M$ Monte-Carlo samples are generated by RKS.  As an example, for the ``Buzz prediction on Twitter'' dataset, our method reduces the test error of RKS by $23.63\%$. However, this accuracy comes at a computational cost in the pre-processing stage. Table \ref{table2} also tabulates the pre-processing, training, and testing time of our algorithm. Except for the ``Online news popularity'' dataset, the pre-processing time is always less than the training time, and most notably, for ``Adult'' dataset it is only $10.2\%$ of the training time. In general, the comparison between the two may not be immediate: our pre-processing requires $O(M_0N_0d)$ computations to evaluate the score function, followed by the time cost of sorting an array of size $M_0$ (on average $O(M_0\log M_0)$). On the other hand, while training requires $O(MNd)$ computations to build the feature matrix, the training time can be largely affected by the choice of regularization parameter used in lieu of the norm-infinity constraint in \eqref{risk-recht}. For instance, in regression, the parameter directly governs the condition number of the $M\times M$ matrix that is to be inverted. As a rule of thumb, one can select $M_0$ in the range $5M$ to $20M$ and set $N_0$ to $0.1N$.

Since our method adds a computational overhead versus RKS, we also plot the generalization error versus time for both methods in Fig. \ref{plot3}. The time for RKS represents the training time, whereas for our method it is the sum of the training and pre-processing time. For example, in Adult dataset, where our computational overhead is quite negligible, our method has superior accuracy versus RKS for any computational time. However, in general the trend is as follows: for very small number of random features (or relatively bad accuracy), our method is inferior, but past a certain computational time threshold, we outperform RKS. This is not surprising: for very small number of features, training is fast, but our method still calculates the empirical score, adding additional cost to training. Once we have more random features, training tends to take more time and the preprocessing time becomes less significant compared to the training time.

\subsection{Linear and Arccosine Kernels}
The feature map $x_{\omega}$ results in linear kernel when $\omega$ is sampled uniformly from $[d]$, and $(\omega^\top x)^nH(\omega^\top x)$ with $\omega \sim\Nc(0,I_d)$ gives the arccosine kernel\footnote{The constraints in Theorem \ref{thm} do not hold for the feature map associated to the arccosine kernel, but we still observed improvement in the generalization error in practice.} of order $n$ ($H(\cdot)$ denotes the Heaviside step function, i.e., $H(x)=0.5+0.5\text{sgn}(x)$, where $\text{sgn}(\cdot)$ is the sign function). These two kernels are among many others that conform to~\eqref{kernel}. In this part, we focus on these two kernels and compare our method with LKRF \cite{sinha2016learning}. ORF and SES are designed for Gaussian kernels and are not applicable in this setting. RKS is data-independent, and as we saw in the previous part, for small number of random features, it is outperformed by EERF and LKRF. 
Table \ref{table3} describes the datasets as well as their corresponding feature map used for the experiment. We follow the previous section in data standardization and tuning the regularization parameters. 

{\noindent \bf Performance of EERF versus LKRF:}
In Fig. \ref{plot2}, we compare our performance with LKRF \cite{sinha2016learning} in terms of the generalization error on several datasets. Our method slightly outperforms LKRF on ``Year prediction'' and ``Letter recognition'', while significantly improving the generalization error on ``Hand movement''. Further, Table \ref{table6} shows the time cost and the generalization error for the largest value of $M$ in the plots. For the pre-processing stage, both algorithms sample an initial distribution $M_0$ times and incur $O\left(dN_0M_0\right)$ computational cost. Our algorithm sorts an array of size $M_0$ with average $O(M_0\log M_0)$ time, while LKRF solves an optimization with $O\left(M_0\log\epsilon^{-1}\right)$ time to reach the $\epsilon$-optimal solution. Therefore, depending on the tolerance $\epsilon$, the processing time may vary for LKRF. 

The main advantage of our method over LKRF is that EERF is parameter-free and does not require tuning. LKRF solves an optimization problem to re-weight random features, which depends on a regularization parameter. The new weights can range from a uniform to a degenerate delta distribution, depending on the regularization parameter, which needs to be tuned. We observed that this brings forward two issues: i) a validation step for tuning the regularization parameter is needed, (ii) the obtained parameter works well only for a range of values for $M$ and needs to be re-tuned for others.

\section{Concluding Remarks}
In this paper, we studied data-dependent random features for supervised learning. We proposed an algorithm called Energy-based Exploration of Random Features (EERF), which is based on a data-dependent scoring rule for sampling random features. We proved that under mild conditions, the proposed score function with high probability recovers the spectrum of best model fit within the postulated model class. We further empirically showed that our proposed method outperforms the state-of-the-art data-independent and data-dependent algorithms based on the randomized-feature approach. The EERF algorithm introduces a small computational pre-processing overhead and requires no additional tuning parameters in contrast to other data-dependent methods for generation of random features. Our method is particularly designed to reduce generalization error in regression and classification. Inspired by the recent results on the application of random features in matrix completion (cf. \cite{si2016goal}), an interesting future direction is to adapt our score function to improve generalization in this setup.

\section{Acknowledgements}
This work was supported in part by NSF ECCS Award No.~1609605. We thank the anonymous reviewers for their constructive feedback and the suggestion to provide time-vs-accuracy plots.

{\small
\bibliography{shahin}
\bibliographystyle{aaai.bst}
}

\newpage

\onecolumn

\section{Supplementary Material}
{\bf Notation:} We denote by $\delta(\cdot)$ the Dirac delta function, by $\ind{\cdot}$ the indicator function, by $\ib=\sqrt{-1}$ the imaginary unit, and by $\mathbf{Card}(\omega)$ the dimension of vector $\omega$, respectively. 

\subsection{Spectrum with Continuous Support}
For the case that $\Omega$ is continuous, we adhere to the following technical assumptions. We associate the subscript ``R'' to the quantities that are ``relevant'' to the recovered part of spectrum, whereas we use the subscript ``I'' for the quantities that are ``irrelevant'' to the the recovered part of spectrum.
\begin{assumption}\label{A1}
Let the spectrum be separable such that $F_0(\omega)=F_I(\omega_I)F_R(\omega_R)$ for $\omega=(\omega_I,\omega_R)$, where $\omega_I\in \Omega_I$, $\omega_R\in \Omega_R$, and $\Omega=\Omega_I\times \Omega_R$. We assume that for $i \in \{I,R\}$, $F_i(\cdot)$ is uniformly bounded, continuous and $L^1$ integrable over $ \Omega_i$, and $x^n$ is an absolutely continuous random variable for all $n \in \{1,\ldots N\}.$
\end{assumption}
\begin{assumption}\label{A2}
The feature map satisfies $\sup_{x \in \Xc, \omega \in \Omega}\abs{\phi(x,\omega)} \leq 1$.  
\end{assumption}
\begin{definition}\label{ortho} (Orthogonality) For any $\omega=(\omega_I,\omega_R)$ and $\omega'=(\omega_I',\omega_R')$, given a parametric distribution $P_{\Xc}^\lambda$ on inputs, we say that the orthogonality condition holds if 
\begin{align}\label{eq444}
\E_{P_{\Xc}^\lambda}[\phi(x,\omega')\phi(x,\omega)]=H_I(\omega_I,\omega_I')H_R(\omega_R-\omega_R';\lambda)+H_I(\omega_I,-\omega_I')H_R(\omega_R+\omega_R';\lambda),
\end{align}
such that $H_I(\cdot,\cdot)$ is uniformly bounded and continuous, $H_R(\cdot;\lambda)\not\equiv 0$ is non-negative for all $\lambda>0$, and 
\begin{align}\label{condition}
 \frac{H_R(\omega_R-\omega_R';\lambda)}{\int_{\Omega_R} H_R(\omega_R-u;\lambda)d_{u}} \longrightarrow \delta(\omega_R-\omega_R'),
\end{align}
as $\lambda \rightarrow \infty$.
\end{definition}
\noindent
We now proceed to the exact statement of Theorem \ref{thm} and its proof. Simply, the theorem indicates that we can recover the part of the spectrum for which the condition \eqref{condition} holds. Note that since we prove the result for a parametric distribution $P_{\Xc}^\lambda$ on $\Xc$, the score function $S(\omega)$ in \eqref{score} becomes $S(\omega;\lambda)$, i.e., it depends on the parameter $\lambda$. 
\begingroup
\def\thetheorem{\ref{thm}}
\begin{theorem}
Let Assumptions \ref{A1}-\ref{A2} hold with $\Omega_R=\R^{\mathbf{Card}(\omega_R)}$, and consider the regression setting, i.e., $\E_{P_{\Yc \vert \Xc}}[y]=f^\star(x)$. Given the orthogonality condition in Definition \ref{ortho}, if $F_R(\cdot)$ is even on $\Omega_R$, for any $\omega_I\in \Omega_I$, the score function $S(\omega;\lambda)=S(\omega_I,\omega_R;\lambda)$ defined in \eqref{score} satisfies 
\begin{equation}\label{eq41}
\lim_{\lambda \rightarrow \infty}\frac{\abs{S(\omega_I,\omega_R;\lambda)-\Ec_{C_0,P_0}(\omega_I,\omega_R;\lambda)}}{\int_{\Omega_R}\abs{S(\omega_I,u;\lambda)-\Ec_{C_0,P_0}(\omega_I,u;\lambda)}d_u}= \frac{\abs{F_R(\omega_R)}}{\int_{\Omega_R}\abs{F_R(u)}d_u},
\end{equation} 
where
$$
\Ec_{C_0,P_0}(\omega_I,\omega_R;\lambda)=\Ec_{C_0,P_0}(\omega;\lambda)\triangleq\E_{P_{\Xc}^\lambda}[(f^\star(x)-f_0(x))\phi(x,\omega)],
$$
is bounded by the sup-norm of $f^\star(\cdot)-f_0(\cdot)$.
\end{theorem}
\addtocounter{theorem}{-1}
\endgroup
\begin{proof}
Recalling the definition of score function in \eqref{score}, we have
\begin{align}\label{eq1}
\resizebox{.94\hsize}{!}{$% 
S(\omega;\lambda)=\E_{P_{\Xc\Yc}^\lambda}[y\phi(x,\omega)]=\E_{P_{\Xc}^\lambda}\E_{P_{\Yc \vert \Xc}}[y\phi(x,\omega)]=\E_{P_{\Xc}^\lambda}[f^\star(x)\phi(x,\omega)]=\E_{P_{\Xc}^\lambda}[f_0(x)\phi(x,\omega)]+\Ec_{C_0,P_0}(\omega;\lambda).
$%
}%
\end{align}
Now, in view of \eqref{bestmodel}, we can substitute $f_0(x)$ into above and get
\begin{align*}
\E_{P_{\Xc}^\lambda}[f_0(x)\phi(x,\omega)]&=\E_{P_{\Xc}^\lambda}\left[\int_\Omega F_0(\omega')\phi(x,\omega')d_{\omega'}\phi(x,\omega)\right]\\
&=\int_\Xc \int_\Omega F_0(\omega')\phi(x,\omega')\phi(x,\omega)d_{\omega'} P_\Xc^\lambda(x)d_x\\
&=\int_\Omega \int_\Xc  F_0(\omega')\phi(x,\omega')\phi(x,\omega) P_\Xc^\lambda(x) d_xd_{\omega'}\\
&=\int_\Omega  F_0(\omega') \left( \int_\Xc \phi(x,\omega')\phi(x,\omega) P_\Xc^\lambda(x)d_x \right) d_{\omega'}, \numberthis \label{eq2}
\end{align*}
as Assumptions \ref{A1} and \ref{A2} warrant valid interchange of integration order. Combining \eqref{eq1} and \eqref{eq2}, we obtain
\begin{align*}
S(\omega;\lambda)-\Ec_{C_0,P_0}(\omega;\lambda)&=\int_\Omega  F_0(\omega') \left(  \int_\Xc \phi(x,\omega')\phi(x,\omega) P_\Xc^\lambda(x)d_x \right) d_{\omega'}\label{eq222}\numberthis\\
&=\int_\Omega  F_0(\omega') \Big(H_I(\omega_I,\omega_I')H_R(\omega_R-\omega_R';\lambda)+H_I(\omega_I,-\omega_I')H_R(\omega_R+\omega_R';\lambda)\Big) d_{\omega'},\numberthis\label{eq11}
\end{align*}
where the last line follows by \eqref{eq444}. Since $F_0(\cdot)$ is uniformly bounded and continuous (Assumption \ref{A1}), we can appeal to Portmanteau Theorem as well as condition \eqref{condition} to get
\begin{align*}
\lim_{\lambda \rightarrow \infty}\frac{S(\omega;\lambda)-\Ec_{C_0,P_0}(\omega;\lambda)}{\int_{\Omega_R} H_R(\omega_R-u;\lambda)d_{u}}&=\int_\Omega  F_0(\omega') H_I(\omega_I,\omega_I')\lim_{\lambda \rightarrow \infty}\frac{H_R(\omega_R-\omega_R';\lambda)}{\int_{\Omega_R} H_R(\omega_R-u;\lambda)d_{u}}d_{\omega'} \\
&+\int_\Omega  F_0(\omega') H_I(\omega_I,-\omega'_I)\lim_{\lambda \rightarrow \infty}\frac{H_R(\omega_R+\omega_R';\lambda)}{\int_{\Omega_R} H_R(\omega_R-u;\lambda)d_{u}} d_{\omega'} \numberthis\label{eq122}\\ 
&=\int_\Omega  F_0(\omega') \Big( H_I(\omega_I,\omega_I')\delta(\omega_R-\omega_R')+H_I(\omega_I,-\omega_I')\delta(\omega_R+\omega_R')\Big) d_{\omega'} \numberthis  \\
&=\vphantom{\int_\Omega} F_R(\omega_R) \int_{\Omega_I} F_I(\omega_I') \Big( H_I(\omega_I,\omega_I')+ H_I(\omega_I,-\omega_I')\Big)d_{\omega_I'} \numberthis\label{eq123},
\end{align*}
where~\eqref{eq123} follows from the fact that $F_R(\cdot)$ is even on $\Omega_R$. Notice by \eqref{eq11} that 
$$
\frac{S(\omega;\lambda)-\Ec_{C_0,P_0}(\omega;\lambda)}{\int_{\Omega_R} H_R(\omega_R-u;\lambda)d_{u}},
$$
is uniformly bounded, and thus by bounded convergence theorem, we have that
\begin{align*}
\lim_{\lambda \rightarrow \infty}\int_{\Omega_R}\abs{\frac{S(\omega_I,z;\lambda)-\Ec_{C_0,P_0}(\omega_I,z;\lambda)}{\int_{\Omega_R} H_R(z-u;\lambda)d_{u}}}d_{z}&=\int_{\Omega_R}\lim_{\lambda \rightarrow \infty}\abs{\frac{S(\omega_I,z;\lambda)-\Ec_{C_0,P_0}(\omega_I,z;\lambda)}{\int_{\Omega_R} H_R(z-u;\lambda)d_{u}}}d_z\\
&=\int_{\Omega_R}\abs{\lim_{\lambda \rightarrow \infty}\frac{S(\omega_I,z;\lambda)-\Ec_{C_0,P_0}(\omega_I,z;\lambda)}{\int_{\Omega_R} H_R(z-u;\lambda)d_{u}}}d_{z}\\
&=\abs{\int_{\Omega_I} F_I(\omega_I') \Big( H_I(\omega_I,\omega_I')+ H_I(\omega_I,-\omega_I')\Big)d_{\omega_I'}}\int_{\Omega_R}\abs{F_R(z)}d_{z}. \label{eq445}\numberthis 
\end{align*}
Now, note that as $\Omega_R=\R^{\mathbf{Card}(\omega_R)}$, it holds that
\begin{align}\label{independence}
\int_{\Omega_R} H_R(\omega_R-u;\lambda)d_{u}=\int_{\Omega_R} H_R(\omega_R+u;\lambda)d_{u}=\int_{\Omega_R} H_R(u;\lambda)d_{u},
\end{align}
implying that in this case the denominator of \eqref{condition} does not depend on $\omega_R$ and is only a function of $\lambda$. The proof follows immediately by taking absolute value from \eqref{eq123}, dividing it by \eqref{eq445}, and applying \eqref{independence}.
\end{proof}
\noindent
Note that for any given $\lambda$, $\Ec_{C_0,P_0}(\omega;\lambda)$ is a decreasing function of $C_0$, because as we see in \eqref{class}, by increasing $C$ we make the class $\Fc_{P_\Omega}$ richer and decrease the projection error.

\noindent
Let us now restrict our attention to the cosine feature map $\phi(x,(\omega_I,\omega_R))=\cos(\omega_R^\top x+\omega_I)$ resulting in functions of the form $$f(x)=\int_\Omega F(\omega_I,\omega_R)\cos(\omega_R^\top x+\omega_I)d_{\omega_I} d_{\omega_R},$$ in \eqref{class}. In this case, $\Omega_R=\R^d$ and $\Omega_I=[0,2\pi)$.

We illustrate that several well-known distributions can be combined with the cosine feature map such that the orthogonality condition in Definition \ref{ortho} holds. As evident from \eqref{kernel}, setting $P_\Omega$ to Gaussian, Cauchy, or Laplacian distribution, the cosine feature map can produce Gaussian, Laplacian, and Cauchy kernels, respectively. However, the focus of below is on the distribution of training samples $P_\Xc$ rather than the distribution of random features $P_\Omega$. 
\begin{corollary}\label{cor:score}
Given Assumptions \ref{A1}-\ref{A2} in Theorem \ref{thm} in the regression setting, consider the cosine feature map $\phi(x,(\omega_I,\omega_R))=\cos(\omega_R^\top x+\omega_I)$  and let the input components be independent. For the cases $x_i\sim\Nc(0,\sigma_i^2)$, $x_i\sim{\tt Laplace}(0,\sigma_i)$, and $x_i\sim{\tt Cauchy}(0,\sigma_i)$, the result of Theorem \ref{thm} holds when $\lambda=(\sigma_1,\ldots,\sigma_d)$ and $\sigma_i \rightarrow \infty$ for all $i \in [d]$.
\end{corollary}
\begin{proof}
It only suffices to show that for each case the orthogonality condition in Definition \ref{ortho} holds. When $\phi(x,(\omega_I,\omega_R))=\cos(\omega_R^\top x+\omega_I)$, we have
\begin{align*}
\phi(x,(\omega_I,\omega_R))\phi(x,(\omega_I',\omega_R'))&=\frac{1}{4}\exp(\ib(\omega_R^\top x+\omega_R'^\top x+\omega_I+\omega_I'))\\
&+\frac{1}{4}\exp(-\ib(\omega_R^\top x+\omega_R'^\top x+\omega_I+\omega_I'))\\
&+\frac{1}{4}\exp(\ib(\omega_R^\top x-\omega_R'^\top x+\omega_I-\omega_I'))\\
&+\frac{1}{4}\exp(-\ib(\omega_R^\top x-\omega_R'^\top x+\omega_I-\omega_I')). \numberthis \label{eq3}
\end{align*}
Therefore, using any distribution on $\Xc$, we need to look up its characteristic function to simplify above. Let us use the notation 
\begin{align*}
\omega_R&=(\omega_{R,1},\ldots,\omega_{R,d})  ~~~~~~~~~~~~~~~~~~~~~~~\omega_R'=(\omega_{R,1}',\ldots,\omega_{R,d}')
\end{align*}
When the $i$-th input component is distributed according to $\Nc(0,\sigma^2_i)$, by standard properties of multi-variate Gaussian distribution, we get 
\begin{align*}
\int_\Xc &\phi(x,(\omega_I,\omega_R))\phi(x,(\omega_I',\omega_R'))P^\lambda_\Xc(x)d_x =\\
&~~~~~~~~\frac{1}{2}\cos(\omega_I+\omega_I')\prod_{i=1}^d\exp\left(-\frac{\sigma_i^2(\omega_{R,i}+\omega_{R,i}')}{2}\right)+\frac{1}{2}\cos(\omega_I-\omega_I')\prod_{i=1}^d\exp\left(-\frac{\sigma_i^2(\omega_{R,i}-\omega_{R,i}')}{2}\right).
\end{align*}
Recalling that $\lambda=(\sigma_1,\ldots,\sigma_d)$, we now have
$$
H_R(\omega_R+\omega_R';\lambda)=\prod_{i=1}^d\exp\left(-\frac{\sigma_i^2(\omega_{R,i}+\omega_{R,i}')}{2}\right) \Longrightarrow \int_{\Omega_R}H_R(\omega_R+\omega_R';\lambda)=\prod_{i=1}^d(\sigma_i/\sqrt{2\pi}),
$$
which implies
\begin{align}\label{limit}
\lim_{\sigma_1,\ldots,\sigma_d \rightarrow \infty} \frac{H_R(\omega_R+\omega_R';\lambda)}{\int_{\Omega_R}H_R(\omega_R+\omega_R';\lambda)}=\delta(\omega_R+\omega_R').
\end{align}
The proof for the other two cases is similar. For the case that the $i$-th input component is distributed according to ${\tt Cauchy}(0,\sigma_i)$, we have 
\begin{align*}
\int_\Xc &\phi(x,(\omega_I,\omega_R))\phi(x,(\omega_I',\omega_R'))P^\lambda_\Xc(x)d_x =\\
&~~~~~~~~~~~~~~\frac{1}{2}\cos(\omega_I+\omega_I')\prod_{i=1}^d\exp\left(-\sigma_i\abs{\omega_{R,i}+\omega_{R,i}'}\right)+\frac{1}{2}\cos(\omega_I-\omega_I')\prod_{i=1}^d\exp\left(-\sigma_i\abs{\omega_{R,i}-\omega_{R,i}'}\right),
\end{align*}
and
$$
H_R(\omega_R+\omega_R';\lambda)=\prod_{i=1}^d\exp\left(-\sigma_i\abs{\omega_{R,i}+\omega_{R,i}'}\right) \Longrightarrow \int_{\Omega_R}H_R(\omega_R+\omega_R';\lambda)=\prod_{i=1}^d(\sigma_i/2).
$$
It is easy to see that \eqref{limit} holds. On the other hand, when the $i$-th input component is distributed according to ${\tt Laplace}(0,\sigma_i)$, we have 
\begin{align*}
\int_\Xc &\phi(x,(\omega_I,\omega_R))\phi(x,(\omega_I',\omega_R'))P^\lambda_\Xc(x)d_x =\\
&~~~~~~~~~~~~~\frac{1}{2}\cos(\omega_I+\omega_I')\prod_{i=1}^d\frac{1}{1+\sigma_i^2(\omega_{R,i}+\omega_{R,i}')^2}+\frac{1}{2}\cos(\omega_I-\omega_I')\prod_{i=1}^d\frac{1}{1+\sigma_i^2(\omega_{R,i}-\omega_{R,i}')^2},
\end{align*}
and 
$$H_R(\omega_R+\omega_R';\lambda)=\prod_{i=1}^d\frac{1}{1+\sigma_i^2(\omega_{R,i}+\omega_{R,i}')^2} \Longrightarrow \int_{\Omega_R}H_R(\omega_R+\omega_R';\lambda)=\prod_{i=1}^d(\sigma_i/\pi).
$$
Again, obtaining \eqref{limit} is immediate.
\end{proof}

\subsection{Spectrum with Discrete Support}
The support set of random features $\Omega$ was assumed to be continuous in the previous section. We now discuss the feature map $\phi(x,\omega)=x_\omega$ resulting in functions of the form $$f(x)=\sum_{\omega=1}^d F_\omega x_\omega,$$ in \eqref{class}. This feature map leads to linear kernel when $\omega$ is drawn uniformly from the set $\Omega=[d]$. In this case, $\Omega$ is discrete, and we require less restrictive assumptions, compared to Theorem \ref{thm}. In particular, we can recover the entire spectrum, which is the vector $F_0$, and do not need the distribution on $\Xc$ to be parametric. We use the notation $F_0=(F_{0,1},\ldots,F_{0,d})$.
\begin{assumption}\label{A3}
We assume that $\norm{F_0}_\infty$ is uniformly bounded, and $x^n$ is an absolutely continuous random variable for all $n \in \{1,\ldots N\}.$
\end{assumption}
\begin{proposition}\label{cor:score2}
Consider the feature map $\phi(x,\omega)=x_\omega$ with $\omega \in [d]$. Given Assumption \ref{A3}, let the input components be independent. If for any $i\in [d]$, $x_i$ is sampled from a distribution with mean zero and variance $\sigma^2$, for any $\omega \in [d]$, we have $$\frac{\abs{S(\omega)-\Ec_{C_0,P_0}(\omega)}}{\sum_{u=1}^d\abs{S(u)-\Ec_{C_0,P_0}(u)}}=\frac{\abs{F_{0,\omega}}}{\sum_{u=1}^d\abs{F_{0,u}}}.$$
\end{proposition}
\begin{proof}: For $\phi(x,\omega)=x_\omega$, we have that
$$
\int_\Xc \phi(x,\omega)\phi(x,\omega')P_\Xc(x)d_x=\int_\Xc x_\omega x_{\omega'}P_\Xc(x)d_x=\sigma^2 \ind{\omega=\omega'},
$$
since the input components are independent, and the $i$-th input component has mean zero and variance $\sigma^2$. Thus, writing the discrete analog of \eqref{eq222} and substituting above into it, we have
$$
S(\omega)-\Ec_{C_0,P_0}(\omega)=\sum_{\omega'=1}^d  F_{0,\omega'} \left( \int_\Xc \phi(x,\omega')\phi(x,\omega) P_\Xc(x)d_x \right)=\sigma^2\sum_{\omega'=1}^d   F_{0,\omega'} \ind{\omega=\omega'}=\sigma^2 F_{0,\omega}.
$$ 
Then, the result of the proposition follows immediately.
\end{proof}

\end{document}